\documentclass[journal]{IEEEtran}
\usepackage{amsmath,epsfig,verbatim,amsthm}
\usepackage{setspace}
\usepackage{amsfonts}

\usepackage{graphicx}
\usepackage[lofdepth,lotdepth]{subfig}
\usepackage{comment}

\def\Z{{\mathbb Z}}

\ifCLASSINFOpdf
\else
\fi
\begin{document}
\title{An analysis of $N\!K$ and generalized $N\!K$ landscapes }

\author{Jeffrey~Buzas
        and~Jeffrey~Dinitz% <-this % stops a space
\thanks{J. Buzas and J. Dinitz are with the Department
of Mathematics and Statistics, University of Vermont, Burlington,
VT, 05401 USA e-mail: (see http://www.cems.uvm.edu/~jbuzas).}% <-this % stops a space
% <-this % stops a space
\thanks{Manuscript February 14, 2013;}}


\markboth{ }%
{Shell \MakeLowercase{\textit{et al.}}: Bare Demo of IEEEtran.cls for Journals}



\newcommand{\NK}{N\!K}
\newtheorem{prop}{Proposition}
\newtheorem{result}{Result}
\newtheorem{example}[prop]{Example}

\maketitle

\begin{abstract}
Simulated landscapes have been used for decades to evaluate search strategies whose goal is to find the landscape location with maximum fitness. Applications  include modeling the capacity of enzymes to catalyze reactions and the clinical effectiveness of medical treatments.
Understanding properties of landscapes is important for understanding search difficulty.
This paper presents a novel and transparent characterization of $\NK$ landscapes.

We prove that $\NK$ landscapes can be represented by parametric linear interaction models where model coefficients have meaningful interpretations. We derive the statistical properties of the model coefficients, providing insight into how the $\NK$ algorithm parses importance to main effects and interactions.   An important insight derived from the linear model representation is that the rank of the linear model defined by the $\NK$ algorithm is correlated with the number of local optima, a strong determinant of landscape complexity and search difficulty. We show that the maximal rank for an $\NK$ landscape is achieved through epistatic interactions that form partially balanced incomplete block designs.  Finally,  an analytic expression representing the expected number of local optima on the landscape is derived, providing a way to quickly compute the expected number of local optima for very large landscapes.
\end{abstract}


% \fi
\IEEEpeerreviewmaketitle

\section{Introduction}
\IEEEPARstart{S}{imulated}
landscapes have been used for several decades to evaluate search strategies whose goal is to find the landscape location with maximum fitness. Applications of simulated landscapes include modeling the capacity of enzymes to catalyze reactions or ligands to bind to proteins, and the clinical effectiveness of medical treatments~\cite{kauffman2},~\cite{eppstein}.
Understanding properties of landscapes is important for understanding search difficulty.

 $\NK$ landscapes are defined by a straightforward, tunable algorithm  where $N$ specifies the number of binary features or loci and $K$ the degree of epistatic interactions among the loci~ \cite{kauffman1}.
$\NK$ landscapes are convenient because there are only two tunable parameters ($N$ and $K$), and yet they provide a very rich set of landscapes.
  There is a large literature exploring properties of $\NK$ landscapes, though there are few analytic results, with~\cite{evans} and~\cite{kaul} notable exceptions.

This paper provides a novel perspective on $\NK$ landscapes and generalizations of $\NK$ landscapes,  proving that these landscapes can be characterized by parametric linear models  comprised of main effects and interaction effects where the model coefficients have meaningful interpretations. The $\NK$ algorithm induces a statistical distribution on the parametric model coefficients.  We derive the distribution of model coefficients, showing how the $\NK$
algorithm, for $K\ll N$, automatically assigns the largest expected magnitude to main effects, with the expected magnitude of interaction effects typically decreasing with increasing order of interaction.

The linear model representation of the landscape suggests that model rank should provide a measure of landscape complexity.    A simple method of assessing rank is provided, and we determine conditions on $N$ and $K$ sufficient for the existence of designs that achieve maximum rank. We then show that the expected number of local optima is proportional to an orthant probability, which can be calculated with reasonable speed and accuracy for very large landscapes. Interestingly and surprisingly, rank is both positively {\it and } negatively correlated with the number of local optima.  For fixed $N$, it is well-known \cite{eremeev} that the number of local optima increases with $K$ (i.e. a positive correlation with landscape rank). We show that when $N$ and $K$ are both fixed, there is a strong {\it negative} correlation between rank and number of local optima for classic $\NK$ landscapes.  The statistical distribution of model effect coefficients provides an explanation for this counter intuitive phenomenon.

  In related work, a proposal to use linear models with main effects and interactions to construct landscapes was explored in~\cite{reeves1},~\cite{reeves2} and~\cite{reeves3}. These authors examined the effect of epistatic interactions on the properties of landscapes using metrics common in the experimental design literature, and they noted the equivalence between Walsh function decompositions and interaction models.   An analysis of $\NK$ landscapes using Walsh functions was given in  \cite{heckendorn}.

\section{Generalized $\NK$ Landscapes and Interaction Models}
\label{sec:generalizednk}

In this section we define the $\NK$ landscape and interaction models.    The models are represented as linear models in matrix form, as this representation facilitates the study of model properties.  For simplicity, the model is first discussed for $K$ constant across loci.  The model is then generalized to allow varying $K$ across loci.

\subsection{$\NK$ landscapes}

A general landscape is defined by a triple $(\chi,d,f)$ where $\chi$ is a set of locations, $d$ is a distance measure and
$f: \chi \rightarrow \mathbb{R}$ is a ``fitness" function.
$\NK$ landscapes are a map from $\chi=(\{0,1\})^N$ to $\mathbb{R}$ where the fitness function $f$ is built from $N$ binary loci and epistatic interactions formed between each locus and $K$ other loci where $K$ can range from $0$ to $N-1$.

To describe $\NK$ landscapes more fully, for each $i=1,\dots,2^N$ let $\mathbf{x_i}\in (\{0,1\})^N$ represent $i$ written in base 2 and represented as a (binary) vector of length $N$.   For $i=1,\dots,N$, let $\mathbf{w_i}$ denote a $2^{K+1}\times 1$ vector of independent random weights, where each component has mean $\mu/N$ and variance $\sigma^2/N>0$ and otherwise arbitrary probability distribution.
 Let $\mathbf{f_i}(\mathbf{x})$ denote a function from $(\{0,1\})^N\rightarrow E^{2^{K+1}}$ where $E^{2^{K+1}}$ denotes unit coordinate vectors in $\mathbb{R}^{2^{K+1}}$, i.e. vectors of the form $(0,0,1,0,\dots,0)$.   The definition of the function $\mathbf{f_i}(\mathbf{x})$ depends on the epistatic interactions to the $i$th locus.  The function is defined explicitly below.

For $j=1,\dots, 2^N$, define $p_j=\sum_{i=1}^N \mathbf{w_i}^T \mathbf{f_i}\mathbf{(x_j)}$.  Note that $p_j$ is the landscape fitness at location $\mathbf{x_j}$ and that $E[p_j]=\mu$ and $\mbox{Var}(p_j)=\sigma^2$ as each $p_j$ is comprised of a sum of $N$ independent weights. The $\NK$ landscape is defined by the fitness-location pairs $(p_j,\mathbf{x_j})$ coupled with Hamming distance.

\subsection{Matrix representation of generalized $\NK$ landscapes  }
\label{subsec:matrixnk}

  We begin by providing an explicit description for the construction and representation of $\NK$ and generalized $\NK$ landscapes as linear models in matrix form. For $i=1,\dots, N$, let $V_i=\{i_1,i_2,\dots,i_{K_i+1}\}$  where $\{i_1,i_2,\dots,i_{K_i+1}\} \subset \{1,2,\dots,N\}$ and $i_1<i_2<\cdots <i_{K_i+1}$ with $i_j=i$ for some $j$. $V_i$ denotes the $i$th interaction set, comprised of locus $i$ and $K_i$ loci that interact with locus $i$. Note that $K$ isn't restricted to be constant across loci, and there are no restrictions on the number of times locus $i$ can appear in the interaction sets.  This represents a generalization of $\NK$ landscapes, with the classic $\NK$ landscape occurring as the special case when $K_i\equiv K$ and each locus appears in exactly $K+1$ interaction sets.  An additional generalization would
  be to not require locus $i$ to be a member of $V_i$.  The results in this article would still hold for this generalization.

  With a slight abuse of notation, for each $i$ write $\mathbf{f_i}\mathbf{(x)}=\mathbf{f_i}(x_{i_1},\dots,x_{i_{K_i+1}})$ where $x_{i_1},\dots,x_{i_{K_i+1}}$ are the elements of $\mathbf{x}$ corresponding to $V_i$.
  %$x_{i_j}$ denotes  the $i_j$th element of $\mathbf{x}$..
  Thinking of $\{x_{i_1},\dots,x_{i_{K_i+1}}\}$ as a binary number, let $e_i(x_{i_1},\dots,x_{i_{K_i+1}})$ be the decimal representation of this number. Define the $1\times 2^{K_i+1}$ vector $\mathbf{f_i}(x_{i_1},\dots,x_{i_{K_i+1}})=(0,\dots,0,1,0,\dots,0)$ where the 1 occurs in column $e_i(x_{i_1},\dots,x_{i_{K_i+1}})+1$.
  Alternatively, if $I_{2^{K_i+1}}$ is the identity matrix of size $2^{K_i+1}$, then $\mathbf{f_i}(x_{i_1},\dots,x_{i_{K_i+1}})$ is the $(e_i(x_{i_1},\dots,x_{i_{K_i+1}})+1)$th row of $I_{2^{K_i+1}}$.

To write the generalized $\NK$ model in matrix form, consider the $1\times C $ vector $\mathbf{f(\mathbf{x})}=\mathbf{f_1}(\mathbf{x})\mid \mathbf{f_2}(\mathbf{x})\mid \allowbreak\dots\mid \mathbf{f_N}(\mathbf{x})$ where
$C= \sum_{i=1}^N2^{K_i+1}$ and $\mid$ denotes column concatenation.  Define the $2^N\times C $ {\it model matrix}
\begin{equation}\label{nkmatrix}
F=\left( \begin{array}{ccc}
\mathbf{f(\mathbf{x_1})} \\
\mathbf{f(\mathbf{x_2})} \\
\vdots \\
\mathbf{f(\mathbf{x_{2^N}})}
\end{array} \right).
\end{equation}
With this definition, the generalized $\NK$ landscape is $\mathbf{p}=F\mathbf{w}$ where $\mathbf{w}=(\mathbf{w_1}^T\mid\dots\mid\mathbf{w_N}^T)^T$ denotes the $C\times 1$ vector of independent random fitnesses.

\textbf{Example:} To illustrate the definition of $\mathbf{f_i}$ and the matrix $F$, suppose $N=3$, $K_i=1$ for $i=1,2,3$ and $V_1=\{1,2\}$, $V_2=\{2,3\}$ and $V_3=\{1,3\}$.  Then

\[ \mathbf{f}_1(x_1,x_2) = \left\{ \begin{array}{llll}
         (1,0,0,0) & \mbox{if $x_1=0 , x_2= 0$};\\
        (0,1,0,0) & \mbox{if $x_1=0 , x_2= 1$};\\
        (0,0,1,0) & \mbox{if $x_1=1 , x_2= 0$};\\
        (0,0,0,1) & \mbox{if $x_1=1 , x_2= 1$}.\end{array} \right. \]
The functions $\mathbf{f}_2(x_2,x_3)$ and $\mathbf{f}_3(x_1,x_3)$ are defined similarly.  The matrix $F$ is then

\smallskip
\[
F=
\left (\begin{array}{rrrrrrrrrrrr}
1&0&0&0&1&0&0&0&1&0&0&0\\
1&0&0&0&0&1&0&0&0&1&0&0\\
0&1&0&0&0&0&1&0&1&0&0&0\\
0&1&0&0&0&0&0&1&0&1&0&0\\
0&0&1&0&1&0&0&0&0&0&1&0\\
0&0&1&0&0&1&0&0&0&0&0&1\\
0&0&0&1&0&0&1&0&0&0&1&0\\
0&0&0&1&0&0&0&1&0&0&0&1\\
\end{array} \right).
\]

\smallskip

To our knowledge, $\NK$ landscapes have never been formalized using the matrix representation given here. The rank of $F$ is a measure of the richness of the landscape, as it gives the dimension of the domain for the vector $\mathbf{p}$.   We will show that the rank of $F$ is determined by $N$, $\{K_i\}_{i=1}^N$ and the structure of the interaction sets $\{V_i\}_{i=1}^N$, and that using rank as a measure of complexity provides refinement beyond using only $N$ and $\{K_i\}_{i=1}^N$.

\subsection{Interaction Model   }
\label{sec:interactmodel}

Here we define the general form of parametric interaction models, and in the next section relate them to $\NK$ landscapes. Statisticians have long employed interaction models to study the effects of multiple `treatments' on an outcome, see for example~\cite{montgomery}.
Interaction models are straightforward to define and model parameters have meaningful interpretations.
In the evolutionary computing literature, these models seem to have received little attention, with the exception of~\cite{reeves1},\cite{reeves2} and~\cite{reeves3}.

To define a general interaction model, for $k=1,\dots,L$, let $U_k\subset\{1,2,\dots,N\}$ and define $\tilde x_i=2x_i-1$.   Mathematical properties of the interaction model and interpretation of model parameters are most easily obtained using the transformed values $\tilde x_i \in \{-1,1\}$ where as before $x_i\in\{0,1\}$.

The general form of an interaction model with $L$ terms is

\[
q(\mathbf{\tilde x}) = \sum_{k=1}^{L}\beta_{U_k}\prod_{j\in U_k} \tilde x_j
\]
where the $\beta$'s are coefficients that can take any value in $\mathbb{R}$ and where we adopt the convention that when $U_k=\emptyset$, $\prod_{j\in U_k}\tilde x_j\equiv 1$.

\bigskip

\textbf{Example:} Consider a model with $N=3$ loci and $U_1=\emptyset$, $U_2=1, U_3=2, U_4=3, U_5=\{1,2\},U_6=\{1,3\}, U_7=\{2,3\}$. The interaction model is $q(\mathbf{\tilde x})=\beta_\emptyset+\beta_1\tilde x_1 +\beta_2\tilde x_2
+\beta_3\tilde x_3+\beta_{12}\tilde x_1\tilde x_2+\beta_{13}\tilde x_1\tilde x_3+\beta_{23}\tilde x_2\tilde x_3$. In this example, $\beta_1,\beta_2$ and $\beta_3$ are main effects coefficients while $\beta_{12},\beta_{13}$ and $\beta_{23}$ are two-loci interaction coefficients.

 The intercept ($\beta_\emptyset$) and coefficients of the main effects and interaction terms have meaningful interpretations. The intercept coefficient represents the average of the fitness values across the entire landscape.  For the general interaction model, it is not difficult to show that the main effect $2\beta_i$ represents the difference in fitness values when locus $i$ is varied between $\tilde x_i=1$ and $\tilde x_i=-1$, averaged over the values of the other $N-1$ loci.   $4\beta_{ij}$ represents the difference of differences between fitness values for $\tilde x_i=1$ and $\tilde x_i=-1$ when $\tilde x_j=1$ and $\tilde x_j=-1$, averaged over all other $\tilde x_k$ for $k\ne i,j$.
In general, higher order interaction coefficients are interpreted as average differences between lower order interactions.  For example, $8\beta_{ijk}$  represents the  average difference in the two factor interaction between loci $i$ and $j$ when $\tilde x_k=1$ and $\tilde x_k=-1$.

\section{Generalized $\NK$ Landscapes as Interaction Models}
\label{sec:nkinteract}

Here we show that generalized $\NK$ landscapes can be expressed as linear interaction models, and that the interactions that are included in the model are completely  determined by the interaction sets $V_i$, $i=1,\dots, N$. We also derive the statistical properties of the interaction model coefficients and show how to construct classic $\NK$ landscapes
that maximize the number of interaction terms.

Let $2^{V_i}$ denote the power set for $V_i$, i.e. the set $V_i$ and all
it's subsets, including the empty set. Let $T=\bigcup_{i=1}^N2^{V_i}$, and consider the interaction model
\begin{equation}\label{intmodel}
p(\mathbf{\tilde x})=\sum_{U\in T }\beta_U\prod_{j\in U}\tilde x_j.
\end{equation}
  Evaluated at the $2^N$ transformed values $\mathbf{\tilde x_1},\dots,\mathbf{\tilde x_{2^N}}$, the model can be written in matrix notation as $\mathbf{p}=\tilde F\mathbf{\beta}$ where $\tilde F$ is an appropriately defined $2^N\times L$ matrix, $L$ is the number of elements in $T$ and $\mathbf{\beta}$ is an $L\times 1$ vector of coefficients.  The random vector $\mathbf{\beta}$ of main effects and interactions has distributional properties dependent on the probability distribution of the vector of weights $\mathbf{w}$ and the structure of the interaction sets $V_i$. The distributional properties of $\mathbf{\beta}$ are studied in section~\ref{subsec:coeffproperties}.  Note that the model defined in \eqref{intmodel} always contains an intercept and all main effects terms.

\bigskip

\textbf{Example}: Consider again the example in Section~\ref{subsec:matrixnk}, where $N=3$, $K_i=1$ for $i=1,2,3$ and $V_1=\{1,2\}$, $V_2=\{2,3\}$ and $V_3=\{1,3\}$.  Then $T=( \{\emptyset\},\{1\},\{2\},\{3\},\{1,2\},\{2,3\},\{1,3\} )$ and

\[\footnotesize
\tilde F=\left ( \begin{array}{rrrrrrr}
1 & -1 & -1 & -1 & 1 & 1 & 1 \\
1 & -1 & -1 &  1 & 1 & -1 & -1 \\
1 & -1 &  1 & -1 & -1 & -1 & 1 \\
1 & -1 &  1 &  1 & -1 & 1 & -1 \\
1 &  1 & -1 & -1 & -1 & 1 & -1 \\
1 &  1 & -1 &  1 & -1 & -1 & 1 \\
1 &  1 &  1 & -1 & 1 & -1 & -1 \\
1 &  1 &  1 &  1 & 1 & 1 & 1 \\
\end{array} \right),
\
\mathbf{\beta}=\left ( \begin{array}{r}  \beta_\emptyset \\
                                         \beta_1 \\ \beta_2 \\ \beta_3 \\ \beta_{12} \\
                                         \beta_{23} \\ \beta_{13} \\
\end{array} \right ).
\]
The first column of $\tilde F$ corresponds to the intercept $\beta_\emptyset$, columns two through four are for the main effects, and columns five through seven are for the two loci interactions.  Column five, for example, corresponds to the interaction between loci 1 and 2, and is obtained by taking the product of the elements of columns two and three, which correspond to the main effects for loci 1 and 2.

The following proposition establishes that a generalized $\NK$ landscape defined by interaction sets $V_1,\dots, V_N$ is equivalent to the interaction model
given by \eqref{intmodel}, thereby clearly establishing the nonzero interaction effects induced by the $\NK$ algorithm.    The proof is given in the appendix.

\begin{prop}\label{prop:nkinteq}
Let $\mathbf{F}$  denote the model matrix for  the generalized $\NK$ landscape model defined in \eqref{nkmatrix},
  and $\mathbf{\tilde F}$ the model matrix for the interaction model defined in \eqref{intmodel}. Then $\mathcal{C}(\mathbf{F})= \mathcal{C}(\mathbf{\tilde F})$ where $\mathcal{C}(\cdot)$ denotes column space.
\end{prop}

Equation \eqref{intmodel} shows that the $\NK$ algorithm constructs an interaction model in an interesting way.  Note that the $\NK$ algorithm dictates that the interaction model contain all sub-interactions contained in higher order interactions.  For example, if an $\NK$ landscape has a fourth order interaction between loci $\{1,2,3,4\}$, then it also has the four third-order and six second-order interactions defined by the subsets of the four interacting loci.

However, knowing which interactions have non-zero coefficients is not sufficient to fully understand the structure of $\NK$
landscapes.  A complete understanding requires knowing  the statistical properties of the main effect and interaction coefficients.

\subsection{Induced properties of interaction model coefficients}
\label{subsec:coeffproperties}

While we have established that generalized $\NK$ landscapes can be represented by interaction models and shown how the epistatic interaction sets define the interaction terms
included in the model, additional insight into landscape properties and a complete understanding of the representation  requires knowledge of the distributional properties of the interaction model coefficients.  This is addressed in the next proposition, with proof given in the appendix.

\begin{prop}\label{prop:betaprop}
Consider a generalized $\NK$ landscape defined by interaction sets $V_1,V_2,\dots,V_N$ and given by $\mathbf{p}=Fw=\tilde F\mathbf{\beta}$ with weight vector $\mathbf{w}$ where $E[w_i]=\mu/N$ and Var$[w_i]=\sigma^2/N$.
For $U \in T$, let $\beta_U$ denote the coefficient of the interaction term corresponding to $U$, and $I(\cdot)$ the indicator function, i.e.
\[ I(U\in 2^{V_{i}}) = \left\{ \begin{array}{ll}
         0 & \mbox{if $U\not\in 2^{V_{i}} $};\\
          1 &  \mbox{if $U\in 2^{V_{i}} $}.\end{array} \right. \]

  Then
\[
E\left [ \beta_\emptyset \right ] = \mu
\]
\[
E\left [ \beta_U \right ] = 0, \mbox{\qquad for\qquad} U\neq \emptyset
\]
\[
\mbox{Cov}[\beta_U,\beta_{U^*}]=0, \mbox{\qquad for\qquad} U\neq U^*
\]
and
\[
\mbox{Var} [\beta_U] =  \frac{\sigma^2}{N}\sum_{i=1}^N2^{-(K_i+1)}I(U\in 2^{V_{i}}).
\]
\end{prop}

Some remarks are in order.  First, if the weights $\mathbf{w}$ are normally distributed, then the interaction  model coefficients, which are  linear functions of $\mathbf{w}$, are independent and normally distributed with the indicated means and variances.  For other distributions on $\mathbf{w}$, the exact distribution of the coefficients is an often intractable convolution problem.
However, as both $N$ and $K$ increase, the central limit theorem can be invoked to show that the coefficients will be approximately normally distributed (contrast with approximate normality of the fitnesses, which only requires $N$ large).
Regardless of the distribution of $\mathbf{w}$, the coefficients are uncorrelated.

 For $\mathbf{w}$ normally distributed, $E[|\beta_U|]=\sqrt{\mbox{Var}[\beta_U]2/\pi}$, i.e. the expected magnitude of $\beta_U$ is proportional to it's standard deviation. More generally, the variance represents the expected magnitude of the square of the coefficient ($\mbox{Var}[\beta_U]=E[\beta_U^2]$ for $U\ne \emptyset$).   With this perspective, it is interesting to note how the $\NK$ algorithm assigns importance (magnitude) to the interaction terms.  The expected squared magnitude of $\beta_U$ depends on the frequency with which $U\in  2^{V_i}$.  Then for the classic $\NK$ model where $K$ is constant and $K\ll N$, main effects have the largest expected magnitude,  second order interactions would typically have larger expected magnitude than third order interactions and so on. On the other hand, when $K=N-1$, all coefficients have the same expected magnitude because each power set $2^{V_i}$ contains main effects and interactions of all orders.

 When the coefficients are normally distributed, independence of model coefficients means that, for example, knowing the magnitude of a two loci interaction provides no information on the magnitudes of the corresponding main effects of the loci--the $\NK$ algorithm assigns magnitudes of the coefficients completely independently.

The interaction model representation of a landscape with binary loci is equivalent to the representation given by Walsh functions~\cite{reeves3}.  While $\NK$ landscapes have been studied from the perspective of Walsh functions~\cite{heckendorn}, the above results provide a transparent and explicit analysis of the Walsh coefficients for generalized $\NK$ landscapes, showing exactly which Walsh coefficients are nonzero and the statistical properties of these coefficients.  For example, note that Theorems 4 and 5 in~\cite{heckendorn}  follow immediately from the representation given in equation~\eqref{intmodel}.

An interesting property of generalized $\NK$ landscapes gleaned from Propositions~\ref{prop:nkinteq} and~\ref{prop:betaprop} is that two landscapes defined by different interaction sets can lead to landscapes with the same set of main effects and interactions, i.e. identical matrices $\tilde F$, but the variances of the coefficients of the landscapes can be different, suggesting that the properties of the resultant landscapes may also be different. Consider, for example, the two landscapes defined by the following interaction sets: $V_{1A}=\{1,2,3,4\},V_{2A}=\{2,3\},V_{3A}=\{1,3\},V_{4A}=\{1,3,4\},V_{5A}=\{2,5\}$ and
$V_{1B}=\{1,4\},V_{2B}=\{1,2,3,4\},V_{3B}=\{3\},V_{4B}=\{4\},V_{5B}=\{2,5\}$. Clearly $T_A=\bigcup_{i=1}^N 2^{V_{iA}}=
\bigcup_{i=1}^N 2^{V_{iA}}=T_B$ so that the landscapes have identical column space, but by Proposition~\ref{prop:betaprop},
the variances of the terms are not all identical. Consider, for example, the main effect for locus 1: From
Proposition~\ref{prop:betaprop}, Var$[\beta_{1A}]={7\sigma^2}/{16N}>\mbox{Var}[\beta_{1B}]={5\sigma^2}/{16N}$.

\subsection{Rank of $\NK$ landscapes and maximal rank models}
\label{subsec:maxrank}

 Here we show that the rank of the model matrix for an $\NK$ landscape is determined by the interaction sets $\{V_i\}_{i=1}^N$ that define the landscape.  We derive the maximum achievable rank and describe how interaction sets can be constructed to   maximize rank.

\begin{prop}\label{prop:rankprop}
$\mbox{rank}(\mathbf{\tilde F})={\cal M}\left(T\right)$ where ${\cal M}(\cdot)$ denotes counting measure and
where $T=\bigcup_{i=1}^N 2^{V_i}$.
\end{prop}

\begin{proof}
The result follows immediately from noting that the complete set of linear and interaction terms of all orders, $\{1,x_1,x_2,\dots,x_N,x_1*x_2,\dots,x_{N-1}*x_{N},x_1*x_2*x_3,\dots,x_1*x_2*\cdots *x_N\}$ comprise a Hadamard matrix,
 a linearly independent set of (orthogonal) column vectors.
\end{proof}

It follows from the proof of Proposition~\ref{prop:rankprop} that $\mbox{rank}(\mathbf{\tilde F})\equiv$ $[$number of model
interactions $-(N+1)]$, i.e. model rank is a one-to-one function of the total number of interactions in an $\NK$ landscape.

The next results give an upper bound on the rank of an $\NK$ landscape, and provide conditions under which maximal rank designs exist for classic $\NK$ landscapes.

\begin{prop}\label{prop4}
For the classic $\NK$ landscape ($K$ constant across loci), $\max\{\mbox{rank}(\mathbf{\tilde F})\}\le \min \{2^N,\allowbreak  N2^{K+1}+1-N(K+1)\}$
where the max is over all possible designs for fixed values of $N$ and $K$.
\end{prop}

\begin{proof}
The maximum rank occurs when there are no overlaps in interactions between the inputs to the loci.  In this case each
locus will contribute $\sum_{j=2}^{K+1}{K+1 \choose j}=2^{K+1}-(K+2)$ interactions of order two or higher.
Summing over the $N$ loci and including the intercept and main effects terms gives the result.
\end{proof}

Proposition \ref{prop4} begs the question as to when maximal rank classic landscapes exist. Equation \eqref{intmodel} shows that all classic $\NK$ landscapes
contain a main effect corresponding to each loci.  Then for both $N$ and $K$ fixed, maximizing rank is equivalent to maximizing the number of interactions
in the model.  From the proof of Proposition \ref{prop4}, it is evident that a necessary and sufficient condition is that each set $V_i$ contain unique pairs of loci, as this ensures no redundant two factor interactions, and by extension no redundant higher order interactions.  We now give explicit cases when the maximum rank given in Proposition \ref{prop4} above can be achieved.
The proof of the following proposition is given in the appendix.

\begin{prop}\label{parameters} There exist classic $\NK$ landscapes of maximal rank for the following values of $N$ and $K$.
\begin{enumerate}
  \item $K=2$ and all $N\geq 7$.
  \item $K=3$ and all $N\geq 13$.
  \item $K=4$ and all $N\geq 21$ except possibly $N =22$.
  \item $K=5$ and all $N\geq 31$ except possibly $N =32,33,34$.
  \item $K=6$ and all $N\geq 51$.
  \item $K=7$ and $N  = 57, 64,67,69$ and all $N\geq 71$.
  \item  $K=8$ and $N  = 73,89$ and all $N\geq 91$.
  \item $K=9$ and $N = 91$ and all $N\geq 111$.

\end{enumerate}
\end{prop}

 It is not difficult to extend Proposition \ref{parameters} to larger values of $K$, and the approach for doing so is contained in the proof of the result. It is worth noting, and not difficult to prove, that a necessary condition for existence of maximal rank landscapes is that $N \geq K^2 +K +1$, because otherwise at least one pair of loci must occur together in more than one set $V_i$.

 Proposition \ref{parameters} established the existence of maximum rank $\NK$ landscapes. We now discuss the construction of these landscapes, which is easily achieved through the use of what we term $\NK$ difference sets.  The combinatorial theory underlying these sets is detailed in the proof of Proposition \ref{parameters} given in the appendix.  Table \ref{diffsets} provides the needed difference sets for up to $K=9$.

 To construct a maximal rank classic $\NK$ landscape, recall that an $\NK$ landscape is completely determined by it's interaction sets $\{V_i\}_{i=1}^N$.  Interaction sets resulting in maximal rank designs are defined by incrementing the elements in the difference sets by one (modulo $N$) until $N$ sets have been defined.   This process is illustrated in the following example.

 \textbf{Example}: Consider constructing an $\NK$ landscape achieving maximum rank when $N=7$ and $K=2$.  From Table  \ref{diffsets}, the $\NK$ difference set is $\{0,1,3\}$.  Then define
 $V_1=\{1,2,4\}, V_2=\{2,3,5\}, V_3=\{3,4,6\},V_4=\{4,5,7\},V_5=\{5,6,1\},V_6=\{6,7,2\},V_7=\{7,1,3\}$.  Note that increments are modulo $N$, i.e. increments exceeding $N$ ``wrap around". % For example an increment of $N+m$ becomes $m$.

 \begin{table}[htd]
\caption{\label{diffsets}
  $\NK$ difference sets used for constructing maximum rank designs. The second column gives values of $N$
  such that maximal $\NK$ landscapes exist.}
\begin{center}
\begin{tabular}{ccc}
$K$ & $N$ & $\NK$ Difference Set \\
2 & $N\ge 7$ & $\{0,1,3\}$ \\
3 & $N\ge 13$ & $\{0,1,4,6\}$ \\
4 & $N\geq 23$ & $\{0,2,7,8,11\}$ \\
5 & $N\geq 35$ & $\{0,1,4,10,12,17\}$ \\
6 & $N\geq 51$ & $\{0,1,4,10,18,23,25\}$ \\
7 & $N\geq 71$ & $\{0,4,5,17,19,25,28,35\}$ \\
8 & $N\geq 91$ & $\{0,2,10,24,25,29,36,42,45\}$ \\
9 & $N\geq 111$ & $\{0,1,6,10,23,26,34,41,53,55\}$ \\
\end{tabular}
\end{center}
\end{table}

 The design constructed in the example is shown in Figure~\ref{fig:maxdesign}.  The $i$th column of the figure gives the elements in $V_i$.  Notice that no pair of loci appears together more than once across columns.  The landscape resulting from this design has rank $N2^{K+1}+1-N(K+1)=36$, resulting from an intercept, 7 main effects, 21 two factor interactions and 7 three factor interactions.

Contrast with Figure~\ref{fig:neighbor} which shows the design resulting from choosing adjacent loci for the epistatic interactions.  This design has rank 29, resulting from an intercept, 7 main effects, 14 two factor interactions and 7 three factor interactions.

\begin{figure}[h]
	\centering
	\subfloat[Subfigure 1 list of figures text][Maximal Rank Design]{
	\includegraphics[width=0.4\textwidth]{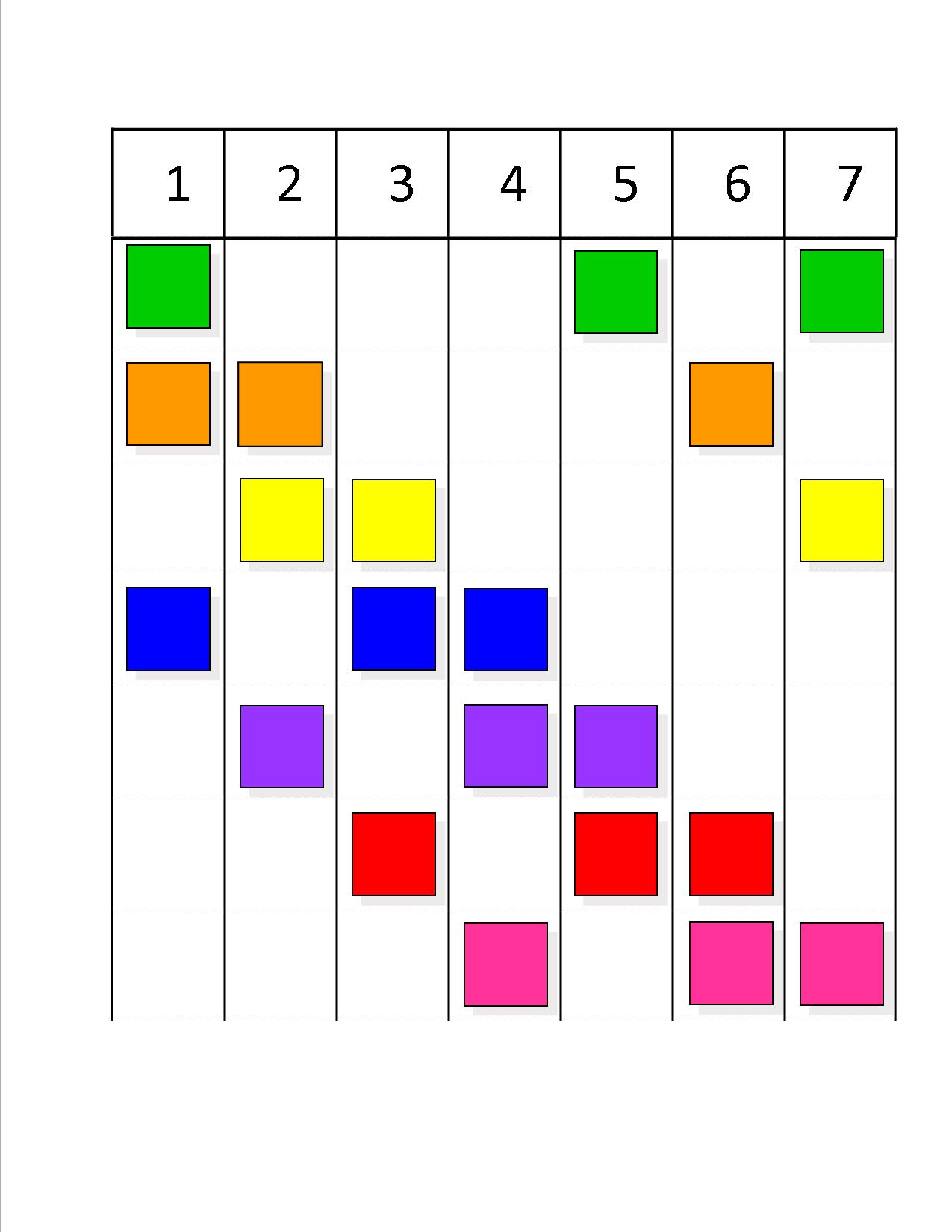}
	\label{fig:maxdesign}}
	\qquad
	\subfloat[Subfigure 2 list of figures text][Adjacent Loci Design]{
	\includegraphics[width=0.4\textwidth]{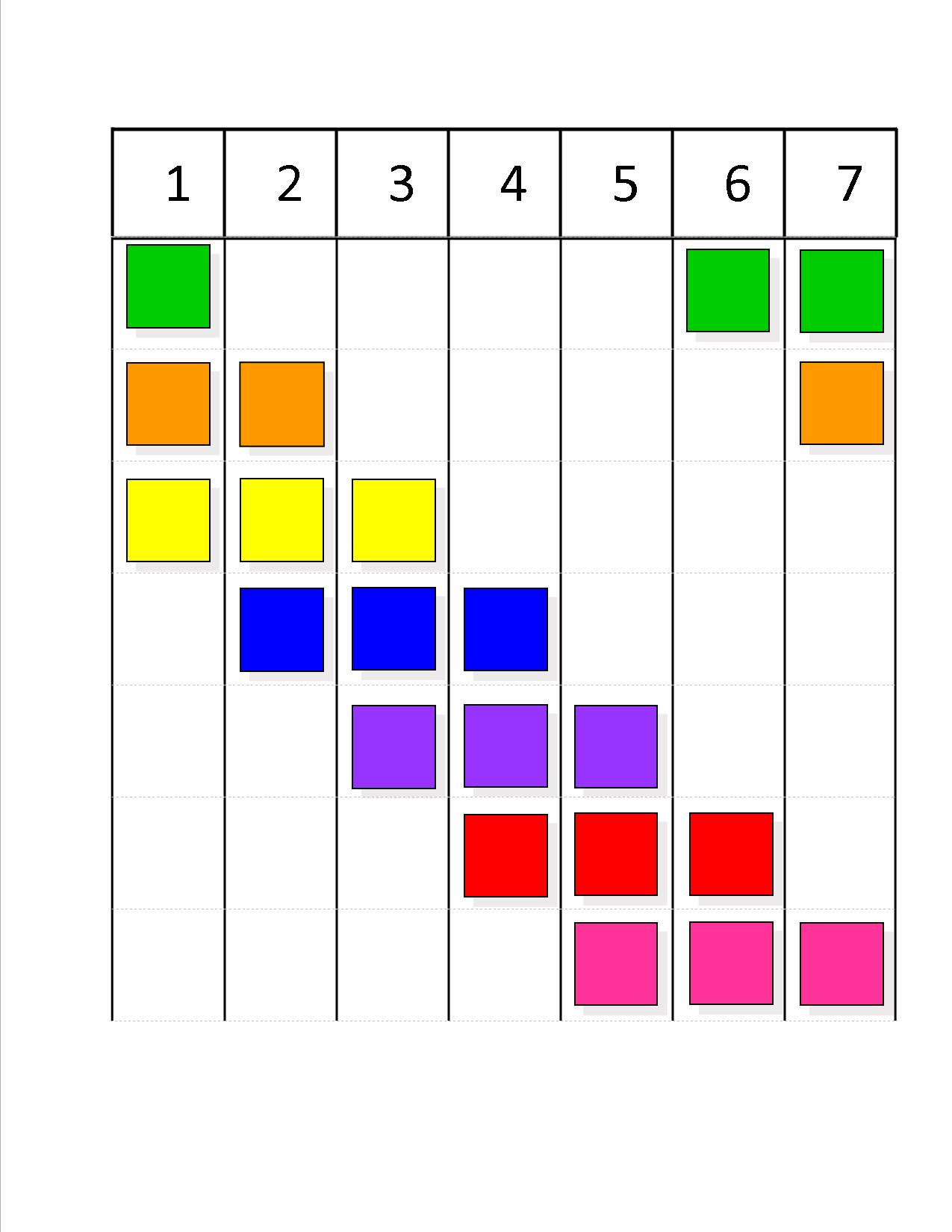}
	\label{fig:neighbor}}
	\caption{Interaction sets for maximal and adjacent loci designs.}
	\label{fig:globfig}
	\end{figure}

\bigskip

\section{Number of Local optima}
\label{sec:localoptima}

The number of local optima is perhaps the strongest measure of landscape ruggedness and search difficulty~\cite{kallel}, \cite{reidys}, \cite{verel}.
For fixed $N$, it has been established empirically that the number of local optima increase with $K$~\cite{eremeev}.  The result is not surprising as increasing $K$ increases
the landscape model rank by increasing both
the number and order of interactions defining the landscape.  An unexplored question is the association between the number of local optima and landscape
rank when $N$ and $K$ are both fixed.

We begin by providing an analytic expression for the expected number of local optima on $\NK$ and generalized $\NK$ landscapes.

\begin{prop}\label{localpeaks}
Consider a generalized $\NK$ landscape defined by interaction sets $V_1,V_2,\dots,V_N$ and weight vector $\mathbf{w}$ where the elements of
 $\mathbf{w}$ are independent normal random variables with $E[w_i]=\mu/N$ and Var$[w_i]=\sigma^2/N$.
Then the expected number of local optima on the landscape is given by $2^N\Phi(\mathbf{0};\Sigma)$ where
\[
\Phi(\mathbf{0};\Sigma)=\int_0^\infty\int_0^\infty\cdots\int_0^\infty e^{-\frac{1}{2}\mathbf{z}^T\Sigma^{-1}\mathbf{z}}dz_1 dz_2\dots dz_N
\]
and where $\Sigma$ is an $N\times N$ symmetric matrix with elements $\sigma_{ij}$ given by
\[ \sigma_{ij} = \left\{ \begin{array}{ll}
         \frac{2\sigma^2}{N}\sum_{k=1}^N I(i\in V_k) & \mbox{if $i=j$};\\
        \frac{\sigma^2}{N}\left (\sum_{k=1}^N I(i\in V_k)+\sum_{k=1}^N I(j\in V_k) \right. \\
       \left. -\sum_{k=1}^N I(i\in V_k\quad\mbox{or}\quad j\in V_k)\right) & \mbox{if $i\ne j$}.\end{array}
        \right. \]
\end{prop}
\begin{proof}
Let $p^*$ denote the fitness of a randomly selected location on the landscape, and let $p^*_i$ denote the fitness one hamming distance away  obtained by flipping the value of locus $i$.  Let $\mathbf{z}=(p^*-p^*_1,p^*-p^*_2,\dots,p^*-p^*_N)$ denote the $N\times 1$ vector of fitness differences.  Then the probability that the random location with fitness $p^*$ is a local optima is given by the orthant probability Pr$(\mathbf{z}>0)$, and the expected number of local optima on the landscape is then $2^N\mbox{Pr}(\mathbf{z}>0)$.  When $\mathbf{w}$ is
jointly normally distributed, it follows that $\mathbf{z}$ is multivariate normal because it is a linear function of $\mathbf{w}$.  The mean of $\mathbf{z}$ is clearly zero, and the variance matrix  is straightforward to derive and is given by $\Sigma$  defined above. The derivation is omitted.
The orthant probability Pr$(\mathbf{z}>0)$ is therefore given by $\Phi(\mathbf{0};\Sigma)$.
\end{proof}

Proposition~\ref{localpeaks} assumes $\mathbf{w}$ is multivariate normal.  The utility of this assumption is that the expected number of local optima
then depends on the multivariate normal orthant probability $\Phi(\mathbf{0};\Sigma)$, allowing us to take advantage of the extensive research
on the numerical computation of multivariate normal probabilities (e.g.~\cite{genz} and references therein). We are then able
to estimate the number of local optima for very large landscapes without having to generate an actual landscape and check individually whether each location is a local peak.
Note the normality assumption implies that the expected number of local optima is a function of only $N$ and $\Sigma$.

For arbitrary probability distributions for the weights $\mathbf{w}$, computation of the orthant probability Pr$(\mathbf{z}>0)$ is typically an intractable $N$ dimensional integral.     When the weights are non-normal and $N$,$K$ are both large, the distribution of $\mathbf{z}$ (see proof of Proposition~\ref{localpeaks})  is approximately normal by the central limit theorem, and normal orthant probabilities should then provide reasonable approximations to the expected number of local optima.

An empirical study was done to explore the relation between landscape rank and the number of local optima for both classic and generalized $\NK$ landscapes.  For each combination of $N=25,50,100$ and $K=1,2,\dots,7$, we first generated interaction sets for 20 classic $\NK$ landscapes. For each value of $N$ and $K$,  maximal rank landscapes were generated when they were known to exist, and an adjacent
 loci interaction design was also generated.  Additional classic $\NK$ landscapes were generated by randomly selecting an $N\times N$ latin square using the methods in~\cite{jacobson}  and then randomly selecting $K+1$ rows from the latin square. The resulting $N$ columns of $K+1$ elements comprise the $N$ interaction sets that define a classic $\NK$ landscape.   For each landscape, the rank and the expected number of local optima were computed.  Multivariate normal orthant probabilities representing the expected number of local optima were computed using the \verb+"mvtnorm"+ package in the R computing environment, see~\cite{genz}.

As seen in Figure~\ref{fig:loglogplotnew}, there is a strong positive correlation between the expected number of local optima and landscape rank for fixed $N$ and increasing $K$.
 Figure~\ref{fig:resolutionplotnew} provides additional resolution with separate plots for $K=3,4,5$ when $N=50$, showing very strong {\it negative} correlation between rank and the number of local optima when $N$ and $K$ are both fixed. Results for $N=25$ and 100 were similar. The maximal or near maximal rank designs are given by the points in the lower right corners of Figure~\ref{fig:resolutionplotnew},
whereas the adjacent designs are given by the point in the upper left corners.  Note that landscapes computed using adjacent loci to define interaction sets resulted in landscapes with significantly lower
rank and larger expected number of local optima than randomly chosen landscapes.

The strong negative correlation seen in Figure~\ref{fig:resolutionplotnew} is perhaps surprising.  Recall that landscape rank is equivalent to the number of terms in the interaction model representation, so that larger rank corresponds to more interaction terms. It is well known that landscapes with main effects but no interactions have only a single peak.
It would seem that additional interaction terms would translate on average to more rugged landscapes, but this is
not the case when $N$ and $K$ are both fixed.

This phenomenon is explained, at least partially, by noting that when
Proposition~\ref{prop:betaprop} is applied to any maximal design for an $\NK$ landscape, it follows that main effects have variance $(K+1)2^{-(K+1)}\sigma^2/N$ and interactions of all orders have variance $2^{-(K+1)}\sigma^2/N$ giving a ratio of $(K+1)/1$, demonstrating that main effects can have considerably more influence than interactions in maximal designs. This observation would clearly extend to designs that are nearly maximal. Conversely, for an adjacent loci design, the variation is spread more equitably among main effects and interactions. For example, applying Proposition~\ref{prop:betaprop},
the ratio of variances for a main effect and two factor interaction is $(K+1)/K$ for the adjacent loci design.

\begin{center}
\begin{figure*}[!t]
\small
\includegraphics[width=7.0in]{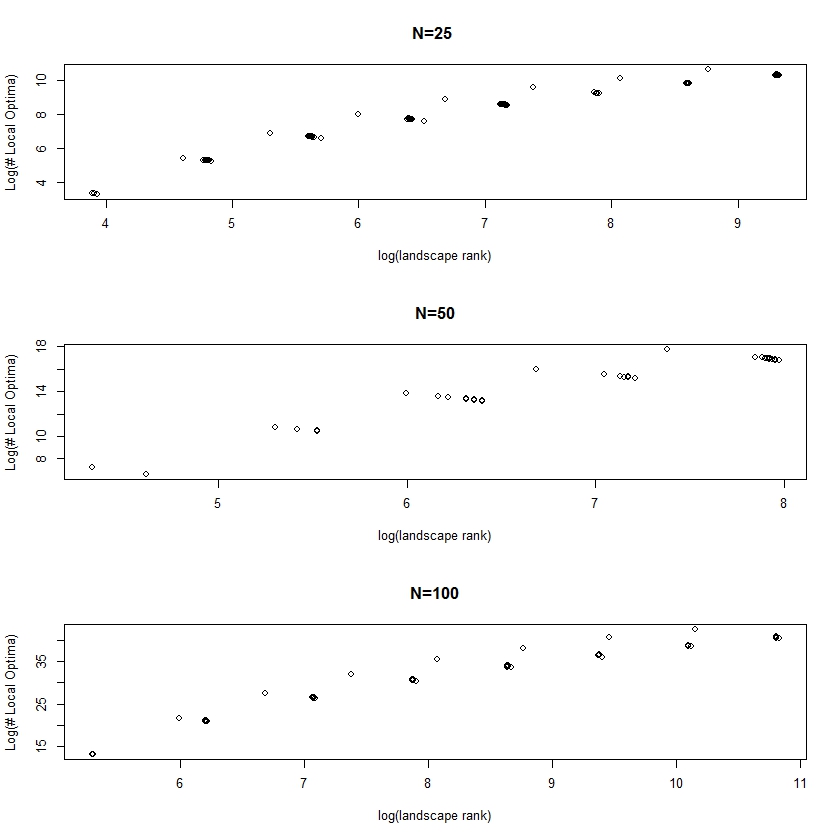} % have to use .pdf files won't recognize .eps or .emf
\caption{Classic $\NK$ landscapes: Expected number of local optima versus landscape rank (on log-log scale) for $K$=1 to 7 }
\label{fig:loglogplotnew}
\end{figure*}
\end{center}

% [width=3.5in] use the width for single column figures.
\begin{center}
\begin{figure*}
\includegraphics[width=7.0in]{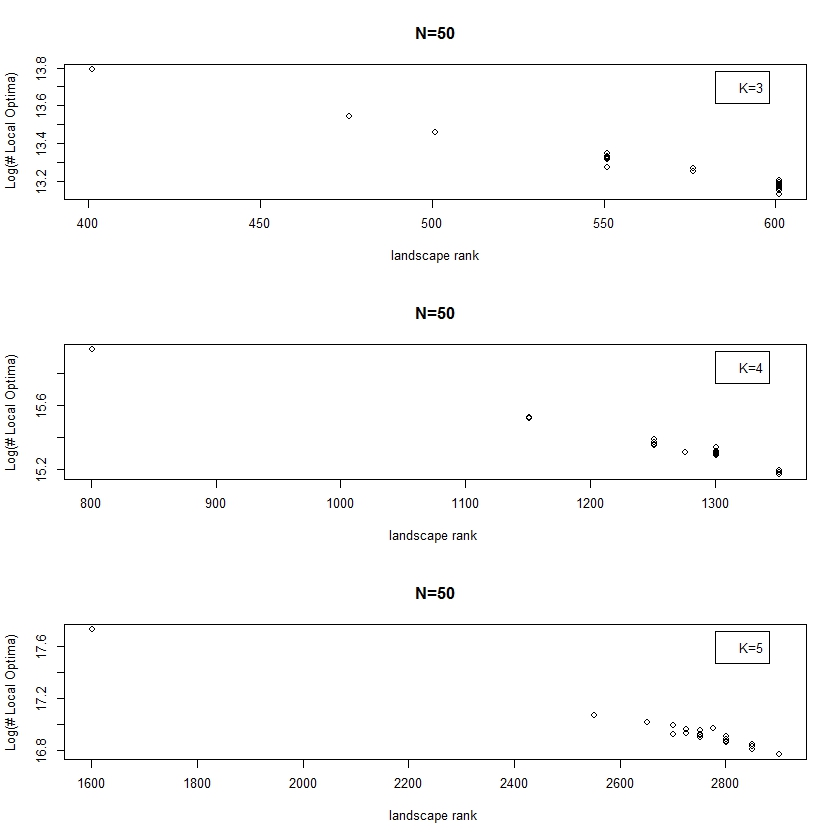} % have to use .pdf files won't recognize .eps or .emf
\caption{Classic $\NK$ landscapes: Expected number of local optima (on log scale) versus landscape rank  for $N=50$, $K$=3,4,5 }
\label{fig:resolutionplotnew}
\end{figure*}
\end{center}

Note also from Proposition~\ref{prop:betaprop} that for fixed $N$ and $K$, the variance of the main effects and the {\it sum} of the variances of interaction coefficients is constant for all classic $\NK$ landscapes, i.e. these quantities
do not vary regardless of how the interaction sets are defined.
It follows that for fixed $N$ and $K$, designs that increase the number of interaction terms must have a decreased average variance for the interaction terms.
 In other words, additional model complexity achieved through the addition of interaction terms is necessarily offset by a reduction in their expected magnitude.

Additional simulations were done to assess the relation between rank and number of local optima for { \it generalized} $\NK$ landscapes.  To construct
generalized $\NK$ landscapes, we kept the sizes of the interaction sets $V_i$ constant for a given landscape,
but did not restrict the number of times that locus $i$ could appear in the interaction sets $V_j$ for $i \ne j$ (for classic $\NK$ landscapes, locus $i$ appears
in exactly $K$  interaction sets $V_j$ for $i \ne j$).

We randomly generated 20 generalized $\NK$ landscapes for each combination of $N=25,50,100$ and $K=1,2,\dots,7$.
A strong positive correlation is again seen between rank and expected number of local optima as $K$ varies with $N$ fixed, see Figure~\ref{fig:loglogplot}.
However, the additional resolution provided by Figure~\ref{fig:resolutionplot} shows a {\it positive} correlation between
landscape rank and expected number of local optima.  An explanation is that for generalized $\NK$ landscapes the expected magnitude of main
and interaction effects can vary, and it is possible to have expected magnitudes for a subset of interactions larger than those for
a subset of main effects.

\begin{center}
\begin{figure*}
\includegraphics[width=7.0in]{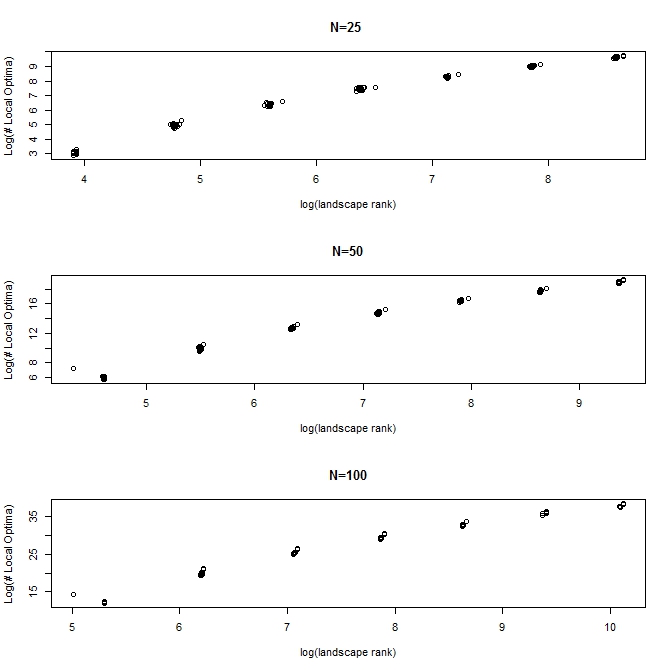} % have to use .pdf files won't recognize .eps or .emf
\caption{Generalized $\NK$ landscapes: Expected number of local optima versus landscape rank (on log-log scale) for $K$=1 to 7 }
\label{fig:loglogplot}
\end{figure*}
\end{center}

\begin{center}
\begin{figure*}
\includegraphics[width=7.0in]{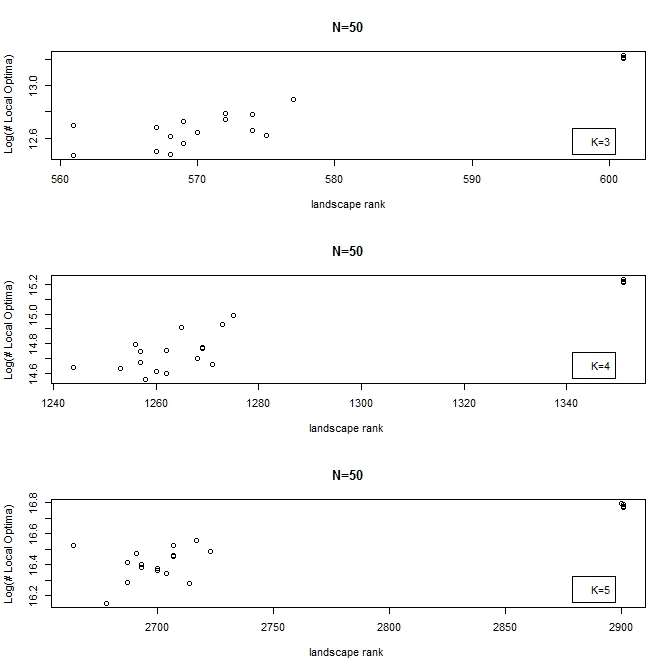} % have to use .pdf files won't recognize .eps or .emf
\caption{Generalized $\NK$ landscapes: Expected number of local optima (on log scale) versus landscape rank  for $N=50$, $K$=3,4,5 }
\label{fig:resolutionplot}
\end{figure*}
\end{center}

\section{Conclusion}

Representation of $\NK$ landscapes as linear models in matrix form provides a transparent connection to parametric linear interaction models, and provides
 a straightforward means for deriving the statistical properties of interaction model
coefficients induced by the $\NK$ landscape algorithm.  The interaction model representation coupled with distributional properties of model coefficients provides  new insights into properties of $\NK$ landscapes.  Expressing  the expected number of local optima as a multivariate normal orthant probability provides additional insight into aspects of $\NK$ landscapes that affect the number of local optima, and also allows for quick computation of the expected number of local optima.

The gaps in the data points seen in the  horizontal axes of Figure~\ref{fig:loglogplot} and Figure~\ref{fig:resolutionplot} represent gaps in the ranks
of $\NK$ landscape models and clearly illustrate that classic $\NK$ landscapes only represent a small subset of possible interaction designs, see also~\cite{heckendorn} for a similar observation. While our definition of generalized $\NK$ landscapes would fill in some of the gaps, the definition still requires that lower order interactions that are contained in higher order interactions appear in the model.

The representation of $\NK$ landscapes as interaction models immediately suggests a more general definition that allows for
main effects and interactions of any order to appear together without restriction.  Such landscapes could be constructed with any rank ranging from 1 to $2^N$, and main effect and interaction coefficients could be generated with arbitrary variances and correlation structures.  In this vein, \cite{reeves1} defined a class of tunable landscapes allowing general interaction structures and studied the performance of GA on these landscapes.

Using arguments similar to those in Section~\ref{sec:localoptima}, it is not difficult to show that for a general interaction model with normally distributed coefficients, the expected number of local optima is represented by a normal orthant probability that depends entirely on the variance/covariance matrix of the vector of fitness differences. However, it is not obvious how to select interactions and assign variances to coefficients that will result in landscapes with large numbers of local peaks, or more generally that are difficult to search.    There are few results regarding the size of normal orthant probabilities as a function of properties of the variance/covariance matrix, though~\cite{rinott} provide potentially useful results in this context.

Characterizing the search difficulty of landscapes is itself difficult~\cite{naudts}, and the number of local optima is an imperfect determinant of search difficulty~\cite{kallel},~\cite{naudts}.  For example, the needle-in-a-haystack landscape, which can be represented as a parametric interaction model containing all possible interactions where the magnitudes of all
main and interactions effects are equal, has a single peak which can only be found by guessing.   While the maximal rank designs for fixed $N$ and $K$ had smaller expected number of local optima, they may
posses other attributes that affect search difficulty.

An attractive feature of the $\NK$ algorithm is that a very rich set of landscapes can be generated with only two tuning parameters. Whether a procedure with comparable simplicity can be developed for constructing a more general class of interaction models that are difficult to search is an area for future research.

\appendices
\appendix[Proof of Propositions~\ref{prop:nkinteq},~\ref{prop:betaprop} and~\ref{parameters}]

The following is a proof of Proposition~\ref{prop:nkinteq}.

\begin{proof}

We first construct a matrix representation for the interaction model. For the $i$th locus, let $S_{i,1}(\mathbf{\tilde x})$ denote the vector of covariates corresponding to the inputs for locus $i$.  For example, if  $V_1=\{1,4,6,9\}$, then $S_{1,1}=(x_1,x_4,x_6,x_9)$. Note that there are $K_i+1$ elements in  $S_{i,1}$.  More generally, let $S_{i,m}(\mathbf{\tilde x})$ denote the set of ${ K_i+1 \choose m}$ $m$-order interactions for the inputs of locus $i$.  For $i=1,\dots N$, define the $1\times 2^{K_i+1}$ vector $\mathbf{\tilde f_i}(\mathbf{x})=(1,S_{i,1},S_{i,2},\dots,S_{i,(K+1)})$, and note that $\mathbf{\tilde f_i}(\mathbf{\cdot})$  can take $2^{K_i+1}$ different values.

%GIVE AN EXAMPLE

 Consider the $1\times C$ vector $\mathbf{\tilde f(\mathbf{x})}=\{\mathbf{\tilde f_1}(\mathbf{x}),\mathbf{\tilde f_2}(\mathbf{x}),\dots,\mathbf{\tilde f_N}(\mathbf{x})\}$ where $C= \sum_{i=1}^N2^{K_i+1}$. Define the $2^N\times C $ {\it model matrix}

\begin{equation}
\tilde F^*=\left( \begin{array}{ccc}
\mathbf{\tilde f(\mathbf{x_1})} \\
\mathbf{\tilde f(\mathbf{x_2})} \\
\vdots \\
\mathbf{\tilde f(\mathbf{x_{2^N}})}
\end{array} \right).
\end{equation}

Note that $\tilde F^*\ne \tilde F$ but clearly $\mathcal{C}(\tilde F^*)=\mathcal{C}(\tilde F)$ as  $\tilde F^*$ is comprised of all the columns in $\tilde F$ with some columns repeated.  $\tilde F^*$ is an overparameterized version of $\tilde F$, and in the following we show that $\mathcal{C}(\tilde F^*)=\mathcal{C}(F)$.

The proof uses two basic results. First, suppose two matrices have the same column space. Choose $l$ rows from the first matrix and form a new matrix by expanding this matrix by row concatenating each of the $l$ rows $j$ times.  The same operations are applied to the second matrix, that is the corresponding $l$ rows from the second matrix are repeated $j$ times. The expanded matrices are then easily seen to have the same column space.  The second result  is that
 if $\mathcal{C}(A_l)=\mathcal{C}(B_l)$ for $l=1,\dots N$ then $\mathcal{C}(A)=\mathcal{C}(B)$ where
$A=A_1\mid A_2 \mid \dots \mid A_N$ and  $B=B_1\mid B_2 \mid \dots \mid B_N$ and where $\mid$ denotes column concatenation.

Let $\mathbf{y_i}=(x_i,x_{i_1},\dots,x_{i_{K_i}})$ and let $\mathbf{y_{i,1}},\mathbf{y_{i,2}},\dots,\mathbf{y_{i,2^{K_i+1}}}$ be the binary ordering of $\mathbf{y_i}$.  Define

\begin{align*}
\tilde a_i &= \begin{pmatrix}
\mathbf{\tilde f_i(\mathbf{y_{i,1}})} \\
\mathbf{\tilde f_i(\mathbf{y_{i,2}})} \\
\vdots \\
\mathbf{\tilde f_i(\mathbf{y_{i,2^{K+1}}})}
\end{pmatrix},  &
a_i &= \begin{pmatrix}
\mathbf{ f_i(\mathbf{y_{i,1}})} \\
\mathbf{ f_i(\mathbf{y_{i,2}})} \\
\vdots \\
\mathbf{ f_i(\mathbf{y_{i,2^{K+1}}})}
\end{pmatrix}.
\end{align*}
Then $\mathcal{C}(\mathbf{a_i})= \mathcal{C}(\mathbf{\tilde a_i})$ as $\tilde a_i$ is a $2^{K_i+1}\times 2^{K_i+1}$ full rank matrix, and $a_i$ is the $2^{K_i+1}\times 2^{K_i+1}$ identity matrix.  Next, let $F_i$ and $\tilde F_i^*$ denote columns
$\sum_{l=1}^{i-1}2^{K_l+1}+1$ to $\sum_{l=1}^{i}2^{K_l+1}$ of $F$ and $\tilde F^*$ respectively.  Note $F_i$ and $\tilde F_i^*$ are obtained from $a_i$ and $\tilde a_i$ by repeating rows of these matrices.  Then by the first result described at the start of the proof, $\mathcal{C}(\mathbf{F_i})= \mathcal{C}(\mathbf{\tilde F_i^*})$.

Finally, the result follows from the fact that $\tilde F^*=\tilde F_1^*\mid \tilde F_2^* \mid \dots \mid \tilde F_N^*$
and $F=F_1\mid F_2 \mid \dots \mid F_N$.

\end{proof}

The following is the proof of Proposition~\ref{prop:betaprop}.

\begin{proof}

Note that $2^N\beta_\emptyset=1^T\tilde F\mathbf{\beta}=1^TF\mathbf{w}$ and the expectation of the RHS is easily seen to be $2^N\mu$ from which it follows that $E[\beta_\emptyset]=\mu$.   Let $\mathbf{\tilde F_{[k]}}$ denote the $k$th column of $\tilde F$ and $\beta_{U_{[k]}}$ the corresponding interaction model coefficient.
Then for $k>1$, $E[\beta_{U_{[k]}}]=E[\mathbf{\tilde F_{[k]}}^T\tilde F\mathbf{\beta}]=E[\mathbf{\tilde F_{[k]}}^T F\mathbf{w}]=0$ where the last equality follows from $\mathbf{1}^T\mathbf{\tilde F_{[k]}}=0$ whenever $k>1$.

To obtain the variance/covariance results, note that $\tilde F\mathbf{\beta}=F\mathbf{w}$ implies
\[
\mathbf{\beta}=2^{-N}\tilde F^TF\mathbf{w}
\]
where we use the fact that the columns of $\tilde F$ are orthogonal, each with norm $2^N$.  Then
\[
\mbox{Var}[\mathbf{\beta}]=2^{-2N}\tilde F^TFF^T\tilde F\times\mbox{Var}[\mathbf{w}]=\frac{\sigma^22^{-2N}}{N}\tilde F^TFF^T\tilde F.
\]

To evaluate $\tilde F^TF$, let $\mathbf{f_{i,j}}$ represent the $j$th column of $F_i$ where $F_i$ represents the submatrix of $F$ corresponding to the interaction set $V_i$ (columns $\sum_{l=1}^{i-1}2^{K_l+1}+1$ to $\sum_{l=1}^{i}2^{K_l+1}$ of $F$). Then it is not difficult to show that
\begin{equation}\label{matrixprod}
\mathbf{\tilde F_{[k]}^T f_{i,j}} =
%\begin{cases}
    h(i,j,k) 2^{N-(K_i+1)} % & \text{if } U_{[k]}\subset V_i\\
%    0               & \text{otherwise}
%\end{cases}
\end{equation}
where the function $h(i,j,k)\in \{-1,1\}$ when $U_{[k]}\in 2^{V_i}$ and zero otherwise ($h(i,j,k)$ is defined explicitly below), and
where $U_{[k]}$ denotes the interaction term corresponding to $\mathbf{\tilde F_{[k]}}$.  Let $\mathbf{g_k}^T$ denote the $k$th row of $\tilde F ^T F$. Then it follows from the identity above that $\mathbf{g_k}^T\mathbf{g_k}=\sum_{i=1}^N 2^{K_i+1}\left(2^{N-(K_i+1)}\right)^2I(U_{[k]}\in 2^{V_k})=\sum_{i=1}^N 2^{2N-(K_i+1)}I(U_{[k]}\in 2^{V_k})$, and then that
  \begin{align*}
\mbox{Var} [\beta_{U_{[k]}}] & =  \frac{\sigma^2}{N}\left (\frac{1}{2^N}\right )^2\sum_{i=1}^N2^{2N-(K_i+1)}I(U_{[k]}\in 2^{V_{i}}) \\
& =\frac{\sigma^2}{N}\sum_{i=1}^N2^{-(K_i+1)}I(U_{[k]}\in 2^{V_{i}}).
\end{align*}

Finally, the covariance result will follow provided the rows of $\tilde F^T F$ are orthogonal.  We first define the function $h(i,j,k)$  appearing in \eqref{matrixprod}. For $j=1,\dots,2^{K_i+1}$, let $e_i^{-1}(j-1)$ denote the $1\times 2^{K_{i+1}}$ vector representation of $j-1$ as a binary number.  The elements of $e_i^{-1}(j-1)$, from left to right, correspond to values for $\{x_{i_1},x_{i_2},\dots,x_{i_{K_i+1}}\}$, compare to the definition of $e_i(\cdot)$ in Section~\ref{subsec:matrixnk}.  Define $h(i,j,k)=I(U_{[k]}\in 2^{V_i})\prod_{U_{[k]}} 2(e_i^{-1}(j-1)-1)$ where the right hand side represents the product of the values of the elements of $e_i^{-1}(j-1)$ corresponding to the loci in $U_{[k]}$.  For $k=1$,  we have $U_{[k]}=U_{[1]}=\emptyset$, and then define $h(i,j,1)=1$ for all $i,j$.

Example: Consider again the example in Section~\ref{subsec:matrixnk}, where $N=3$, $K_i=1$ for $i=1,2,3$ and $V_1=\{1,2\}$, $V_2=\{2,3\}$ and $V_3=\{1,3\}$.  The interaction model is $p=\beta_{\emptyset}+\beta_1\tilde x_1+\beta_2\tilde x_2+\beta_3\tilde x_3
+\beta_{12}\tilde x_1\tilde x_2+\beta_{13}\tilde x_1\tilde x_3+\beta_{23}\tilde x_2\tilde x_3$.

For $i=1$, $j=3$, $e_1^{-1}(3-1)=(1,0)$.  Then for $k=2$, $U_{[2]}=\{1\}$ and $h(1,3,2)=1$. For $k=3$, $U_{[3]}=\{2\}$
and $h(1,3,3)=-1$.  For $k=5$, $U_{[5]}=\{1,2\}$ and $h(1,3,5)=1\times -1=-1$. For $k=6$, $U_{[6]}=\{1,3\}$,
$U_{[6]}\not\in 2^{V_1}$ and $h(1,j,6)=0$ for $j=1$ to $4$.

The inner product of row $k$ and $l$ of $\tilde F^T F$ is
 \begin{align*}
\sum_{i=1}^N\sum_{j=1}^{2^{K_i+1}}\mathbf{\tilde F_{[k]}^T f_{i,j}} & \mathbf{\tilde F_{[l]}^T f_{i,j}} = \\
 & \sum_{i=1}^N 2^{2N-2(K_i+1)}  \sum_{j=1}^{2^{K_i+1}}h(i,j,k)h(i,j,l).
\end{align*}
Note that $\sum_{j=1}^{2^{K_i+1}}h(i,j,k)h(i,j,l)=0$ for all $i$ follows from orthogonality of $\mathbf{\tilde F_{[k]}}$
and $\mathbf{\tilde F_{[l]}}$.
Therefore the rows of $\tilde F^T F$ are orthogonal, and it follows that
$\mbox{Cov}(\beta_U,\beta_{U^*})=0$ for $U\neq U^*$.
\end{proof}

The following is a proof of Proposition \ref{parameters}. In proving Proposition \ref{parameters}, we also define and
 demonstrate the properties of the $\NK$ difference sets used to construct maximum rank designs.

 We begin with the notion of a {\em packing design} or just {\em packing} for short.  These are variants of well-known objects from combinatorial design theory, see~\cite{colbourn}; for completeness we give the definition here.

A $(N,\kappa)-${\em packing design}, consists of a set $S$ on $N$ elements (called "points") and a collection of subsets of $S$ (called the "blocks") all of size $\kappa=K+1$ with the property that each pair of points in $S$ is contained in at most one of the blocks.    To obtain an $\NK$ landscape of maximum rank, one can use the points and the blocks of a $(N,\kappa)$ packing design to construct the $N$ interaction sets defining the landscape. The next theorem gives this connection.

\begin{prop}\label{designs-work}
If there exists an $(N,\kappa)$ packing design, then there exists a  classic $\NK$ landscape with $N$ loci and $K=\kappa-1$ which achieves maximum rank.
\end{prop}

\begin{proof}
As noted in the proof of Proposition \ref {prop4}, maximal rank designs will occur when each {\em pair} of loci occurs at most once.  The existence of a $(N,\kappa=K+1)$ packing design ensures that there are no redundant two factor interactions in the $\NK$ interaction set specification, and by extension no redundant interactions of order greater than two.  Hence a $\NK$ landscape of maximal possible rank will result. \end{proof}

In order to construct $\NK$ landscapes via Proposition \ref{designs-work} we must construct $(N,\kappa)$ packing designs.   Constructing $(N,\kappa)$ packing designs is made easier by employing so called difference methods from combinatorial design theory.  Let $\Z_n$  be the integers modulo $n$, so $\Z_n = \{0,1,2,\ldots n-1\}$ with addition modulo $n$.  Define an $(N,\kappa)$ {\em  difference set} in $\Z_N$ to be a subset $D = \{x_1, x_2, \ldots , x_\kappa\} \subseteq \Z_N$  of size $\kappa$ with the property that the list of all differences $x_i -x_j$ (where $x_i,x_j \in D$ and $i \neq j$) contains each nonzero element {\em at most} one time.

An example of a $(8,3)$ {\em  difference set} is $D= \{0,1,3\}$. Note that in $Z_8$, $1-0 =1,\ 0-1 = 7,\ 3-0 = 3,\ 0-3= 5,\ 3-1=2,$ and $ 1-3 = 6$ and hence all of the differences are different as required.

We can use a $(N,\kappa)$ {\em difference set} in $\Z_N$ to construct an $(N,\kappa)$ packing design and thus via Proposition \ref{designs-work} we will have constructed a classic $\NK$ landscape with $N$ loci and $K=\kappa-1$ which achieves maximum rank. The construction is given in the next proposition.

\begin{prop}\label{differences-work}
If there exists an $(N,\kappa)-${\em difference set} in $\Z_N$, then there exists an $(N,\kappa)$ packing design.
\end{prop}

\begin{proof}
Let $D =  \{x_1, x_2, \ldots , x_\kappa\}$ be an $(N,\kappa)$ {\em  difference set} in $\Z_N$.  For $g \in \Z_N$ define $D +g=  \{x_1+g, x_2+g, \ldots , x_\kappa+g\}$. $D+g$ is called a {\em translate} of $D$. We claim the set of all translates of $D$ (namely $\{ D +g: g \in \Z_N\}$) form the blocks of an $(N,\kappa)$ packing design.  First note that there are exactly $N$ of these translates (one for each element in $Z_N$) and also note that there are $N$ points (the elements of $Z_N$).   We must show that any pair of elements $a,b \in \Z_N$ that $a$ and $b$ are in at most one of the translates of $D$.  Assume to the contrary that the pair $\{a,b\}$ occurs in two of the translates, say $\{a,b\} \subseteq D+p$ and $\{a,b\} \subseteq D+q$ where $p \neq q$ are both elements of $\Z_N$.   Then $a = x_i +p$ and $b = x_j +p$ for some $i$ and $j$ and also $a = x_k +q$ and $b = x_m +q$ for some $k$ and $m$.  Hence $a-b = x_i -x_j = x_k -x_m$.  From the difference property of $D$ this implies that $i=k$ and $j=m$ which in turn says that $p=q$, a contradiction.  Hence we have shown that the $n$ translates of an $(N,\kappa)$ {\em difference set} in $\Z_N$ give the blocks of an $(N,\kappa)$ packing design.
\end{proof}

\section*{Acknowledgment}

The authors would like to thank Stuart Kauffman, Jeffrey Horbar and Maggie Eppstein for discussions motivating this work.

\ifCLASSOPTIONcaptionsoff
  \newpage
\fi



%\bibliographystyle{IEEEtran}


\begin{IEEEbiographynophoto}{Jeffrey Buzas}
Jeff Buzas is an associate professor  of Mathematics and Statistics and Director of the Statistics Program at the University of Vermont.
Dr. Buzas received his Ph.D. in statistics at North Carolina State University in 1993.  He recent interest in evolutionary computation
was the result of conversations with computer scientists and physicians on how to model the search for better medical practice and treatments
as a search on fitness landscapes.
\end{IEEEbiographynophoto}

\begin{IEEEbiographynophoto}{Jeffrey Dinitz}
Jeff Dinitz is a professor of Mathematics and Statistics at the University of Vermont and holds a secondary appointment in the Department of Computer Science. He was Chair of the Department of Mathematics and Statistics from 1998 – 2004 and Chair of the Computer Science Department from 2010 – 2012. Dr. Dinitz received his Ph.D. in mathematics at The Ohio State University in 1980 with a specialty in combinatorial designs.  Since 1980, Dr. Dinitz has published extensively in this area and in combinatorics in general with over 90 publications in refereed research journals.  Dr. Dinitz is the managing editor-in-chief of the Journal of Combinatorial Designs  and is also one of two coeditors of the Handbook of Combinatorial Designs.
\end{IEEEbiographynophoto}

\end{document}